\documentclass[journal]{IEEEtran}

\usepackage[T1]{fontenc}
\usepackage{lmodern}
\usepackage[utf8]{inputenc}
\usepackage{amsmath,amssymb,amsthm}
\usepackage{algorithm}
\usepackage{algpseudocode}
\usepackage{graphicx}
\usepackage{booktabs}
\usepackage{cleveref}
\usepackage{tikz}
\usepackage{pgfplots}
\pgfplotsset{compat=1.17}

\theoremstyle{plain}
\newtheorem{theorem}{Theorem}[section]

\newtheorem{corollary}[theorem]{Corollary}
\newtheorem{proposition}[theorem]{Proposition}

\theoremstyle{definition}
\newtheorem{definition}[theorem]{Definition}
\newtheorem{example}[theorem]{Example}

\theoremstyle{remark}
\newtheorem{remark}[theorem]{Remark}

\DeclareMathOperator{\KL}{KL}
\DeclareMathOperator{\Var}{Var}

\newcommand{\E}{\mathbb{E}}
\newcommand{\V}{\mathcal{V}}
\newcommand{\T}{\mathcal{T}}

\begin{document}

\title{Recursive Meta-Distillation: An Axiomatic Framework for Iterative Knowledge Refinement}

\author{Aaron R. Flouro and Shawn P. Chadwick, PhD\\
research@sparse-tech.com}

\maketitle

\begin{abstract}
Recent work in probability-domain knowledge distillation has established axiomatic frameworks for temperature scaling, multi-teacher aggregation, and bias--variance trade-offs in single-stage settings. However, the mathematical behavior of recursive or multi-generation distillation remains poorly understood, with prior approaches relying primarily on empirical heuristics. In this work, we introduce an axiomatic and operator-theoretic framework for recursive meta-distillation, formalizing iterative knowledge distillation as a sequence of probability-distribution operators with explicit anchoring to base teachers.

We define structural axioms for valid meta-teacher construction and prove the existence of non-trivial operator families satisfying these axioms without specifying particular algorithms or loss functions. Under mild realizability and convexity assumptions, we show that anchored recursive distillation induces contraction in KL divergence, yielding geometric convergence to base teacher distributions and a unique, globally attractive fixed point.

The contribution is foundational rather than algorithmic: the framework characterizes when recursive distillation is mathematically well-posed and convergent rather than error-accumulating, independent of model architecture, optimization details, or specific operator instantiations. These results provide a theoretical basis for understanding stability, bias--variance behavior, and failure modes in iterative and multi-teacher distillation under capacity constraints.
\end{abstract}

\begin{IEEEkeywords}
Knowledge Distillation, Recursive Learning, Meta-Teacher, Operator Theory, Fixed-Point Analysis
\end{IEEEkeywords}

\subsection*{Why This Framework Is Needed}
While iterative distillation methods have shown empirical promise, the field lacks a unified theoretical framework addressing fundamental questions: Under what conditions does recursive distillation converge rather than drift? How should previous student generations contribute to meta-teacher construction? What anchoring mechanisms prevent unbounded error accumulation? This paper provides such a framework, with the axiomatic characterization and contraction analysis serving as enabling foundations for rigorous convergence guarantees rather than heuristic guidance.

\section{Introduction}

Knowledge distillation (KD)~\cite{hinton2015distilling} traditionally transfers knowledge from a single teacher to a student in one pass. Recent advances in probability-domain temperature scaling~\cite{paper1} and multi-teacher ensemble distillation~\cite{paper2} have extended this paradigm to settings with multiple heterogeneous teachers contributing to unified students.

This work builds on these foundations, extending the single-pass frameworks of~\cite{paper1,paper1_5,paper2} to iterative, generational refinement. We introduce recursive meta-distillation, a framework for iterative knowledge refinement where:
\begin{itemize}
\item A sequence of student generations is trained successively
\item Each generation learns from a meta-teacher combining base teachers with previous student generations
\item The process progressively moves students closer to base teacher distributions
\end{itemize}

\subsection{Contributions}

Our main contributions are:

\begin{enumerate}
\item Axiomatic Characterization: We define axioms for valid meta-teacher construction operators and prove existence without specifying explicit formulas

\item Contraction Principles: We establish abstract conditions under which distance to base teachers decreases across generations

\item Fixed-Point Analysis: We characterize recursive distillation as an operator with well-defined equilibrium behavior

\item Convergence Theory: We analyze conditions for global convergence to base teacher distributions

\item Bias-Variance Framework: We provide theoretical perspective on how iterative refinement balances bias and variance
\end{enumerate}

\noindent Table~\ref{tab:comparison} summarizes how this framework differs from prior iterative distillation methods: unlike sequential distillation and Born-Again Networks, which offer only empirical observations, recursive meta-distillation provides provable geometric convergence guarantees through explicit teacher anchoring.

\subsection{Organization}

Section~\ref{sec:setup} introduces the formal setup and notation. Section~\ref{sec:axioms} defines axioms for meta-teacher operators. Section~\ref{sec:contraction} establishes contraction principles. Section~\ref{sec:fixedpoint} analyzes fixed-point behavior. Section~\ref{sec:biasvariance} discusses bias-variance perspectives. Section~\ref{sec:conclusion} concludes.

This paper is not an algorithmic proposal, but a foundational framework that characterizes when recursive distillation is mathematically well-posed and convergent. We deliberately omit empirical validation, focusing instead on mathematical well-posedness; empirical evaluation is orthogonal to the theoretical contributions here and is deferred to follow-up work.

Recent work has explored iterative and multi-teacher distillation empirically or algorithmically, including feedback-driven refinement~\cite{iterative_kd_2024}, optimization-theoretic perspectives~\cite{undo_kd_2025}, and adaptive multi-teacher weighting~\cite{mtkd_rl_2025}. Comprehensive surveys document this rapid expansion~\cite{kd_survey_2025}. However, these approaches emphasize empirical performance rather than formal convergence analysis. This paper provides the first axiomatic and operator-theoretic framework that characterizes when recursive distillation is mathematically well-posed and convergent.

This work defines structural conditions under which repeated distillation constitutes a contractive dynamical process on probability distributions rather than an error-accumulating one. By formalizing meta-teacher construction as an abstract operator with explicit anchoring and continuity properties, we prove the existence of non-trivial recursive schemes that converge geometrically to base teacher distributions, characterize their fixed points, and identify failure modes when anchoring is absent. The contribution is not a new algorithm or loss, but a general theory that explains when and why recursive distillation converges, independent of model architecture, optimization details, or specific operator instantiations.

\section{Setup and Notation}
\label{sec:setup}

\subsection{Recursive Distillation Setting}

\begin{definition}[Recursive Distillation Environment]
\label{def:recursive-env}
Let:
\begin{itemize}
\item $\V$ be a finite vocabulary with $|\V| = V$
\item $T_0$ denote the base teacher (single frontier model or multi-teacher ensemble)
\item $S_g$ denote the generation-$g$ student model with parameters $\theta_g$
\item For any model $M \in \{T_0, S_0, S_1, \ldots\}$, the conditional probability distribution over tokens given input $x$ is $p^{(M)}(i|x) \in [0,1]$ with $\sum_i p^{(M)}(i|x) = 1$
\end{itemize}
\end{definition}

\begin{remark}[Positivity Assumption]
Throughout, we assume all model distributions assign strictly positive probability to every token: $p^{(M)}(i) > 0$ for all $M$, $i$. This ensures all KL divergences and logarithms are finite.
\end{remark}

\subsection{Standing Assumptions}

We make the following standing assumptions:

\begin{enumerate}
\item Positivity: All distributions assign strictly positive probability to all tokens
\item Capacity: Students can approximate meta-teacher distributions when trained to convergence
\item Optimization: Training procedures achieve approximate global minima
\item Temperature: Unless stated otherwise, we work in probability domain with unit temperatures ($T = 1$)
\end{enumerate}

\begin{remark}[Temperature Scope Clarification]
\label{rem:temp-scope}
The present analysis assumes unit temperature to preserve linearity of meta-teacher construction and enable operator-theoretic analysis. When non-unit temperatures are introduced, the resulting operators become nonlinear in probability space, and contraction analysis requires additional tools such as nonlinear fixed-point theorems or variational inequalities. Extending the present framework to temperature-scaled operators is a natural direction for future work and would unify this framework with the temperature-scaling analysis of~\cite{paper1}.
\end{remark}

\section{Axiomatic Framework for Meta-Teacher Operators}
\label{sec:axioms}

\subsection{Motivation}

Standard knowledge distillation is a single-stage operation. Recursive meta-distillation extends this to multiple generations, raising questions about:
\begin{itemize}
\item How should earlier students contribute to later generations?
\item How can we prevent unbounded drift from base teachers?
\item Under what conditions does the process converge?
\end{itemize}

We address these questions through axiomatic characterization of valid meta-teacher construction operators.

\subsection{Operator Axioms}

\begin{definition}[Meta-Teacher Construction Operator]
\label{def:meta-teacher-op}
A meta-teacher construction operator is a function $G$ that maps:
\begin{itemize}
\item Base teacher distribution(s): $p^{(T_0)}, \ldots, p^{(T_k)}$
\item Previous student distributions: $p^{(S_0)}, \ldots, p^{(S_g)}$
\item Control parameters: $\alpha$, $\{w_k\}$, $\{v_j\}$
\end{itemize}
to a probability distribution $q_g$ over the vocabulary $\V$.
\end{definition}

Without explicit anchoring to the base teacher, recursive distillation reduces to a self-training loop in which approximation errors accumulate across generations, leading to unbounded drift.

\noindent\textbf{Axiom 1} (Convexity Preservation).
The operator $G$ must produce a valid probability distribution:
\begin{itemize}
\item Non-negativity: $q_g(i) \geq 0$ for all $i \in \V$
\item Normalization: $\sum_i q_g(i) = 1$
\end{itemize}

\noindent\textbf{Axiom 2} (Positivity Inheritance).
If all input distributions assign strictly positive probability to every token, then $q_g(i) > 0$ for all $i$.

\noindent\textbf{Axiom 3} (Teacher Anchoring).
There exists a weight $\alpha > 0$ such that the base teacher $T_0$ influences the meta-teacher with at least weight $\alpha$. The anchor weight $\alpha$ quantifies the minimum influence of the base teacher at each generation and serves as the primary stability parameter governing convergence speed, robustness to optimization error, and resistance to drift. Formally, for some reference distribution $r$ and all $i$:
$$\min(q_g(i), p^{(T_0)}(i)) \geq \alpha \cdot \min(r(i), p^{(T_0)}(i))$$

This ensures every meta-teacher incorporates direct guidance from $T_0$.

\begin{remark}[Sufficient Anchoring Condition]
\label{rem:sufficient-anchor}
A sufficient and commonly used anchoring condition is the \emph{convex mixture} form:
$$q_g = \alpha \, p^{(T_0)} + (1-\alpha) \, \tilde{q}_g$$
where $\tilde{q}_g$ is any valid distribution over $\V$ (e.g., derived from previous students or auxiliary teachers) and $\alpha \in (0,1]$. This form trivially satisfies Axiom~3 and yields the explicit contraction factor $\beta = 1-\alpha$ in Theorem~\ref{thm:contraction}. Axiom~3 is stated in greater generality to accommodate non-convex aggregation schemes and information-theoretic projections that may arise in practice.
\end{remark}

\begin{remark}[Role of the Anchor Weight $\alpha$]
\label{rem:alpha-role}
The anchor weight $\alpha \in (0,1]$ serves as a \emph{stability parameter} governing three key properties: (i)~\emph{convergence speed}--larger $\alpha$ yields faster geometric decay with contraction factor $\beta = 1-\alpha$; (ii)~\emph{robustness to optimization error}--$\alpha$ bounds the neighborhood to which the process converges under imperfect optimization (Remark~\ref{rem:approx-contraction}); and (iii)~\emph{drift prevention}--any $\alpha > 0$ ensures the base teacher signal persists across generations, preventing unbounded error accumulation.
\end{remark}

\noindent\textbf{Axiom 4} (Continuity).
The operator $G$ is jointly continuous in all input distributions and control parameters.

\noindent\textbf{Axiom 5} (Monotonicity in Anchor Weight).
For two meta-teachers constructed with anchor weights $\alpha_1 < \alpha_2$, the meta-teacher with $\alpha_2$ is closer to $p^{(T_0)}$ in KL divergence (when other parameters held fixed).

\subsection{Existence Results}

\begin{theorem}[Existence of Conforming Operators]
\label{thm:existence}
There exist non-trivial operator families $G$ satisfying Axioms~1--5.
\end{theorem}

\begin{proof}
We establish existence via construction principles without specifying exact formulas:

\begin{enumerate}
\item \emph{Weighted averaging approach:} For any collection of distributions and non-negative weights summing to 1, their weighted average satisfies Axioms~1--2 by convexity. Choosing weights to ensure $\alpha > 0$ on $T_0$ satisfies Axiom~3. Continuous dependence on weights yields Axiom~4. Monotonicity follows from convex geometry.

\item \emph{Information-theoretic projection:} Operators based on minimizing weighted information divergence to teacher distributions can be shown to satisfy all axioms under appropriate constraints.

\item \emph{Convex optimization formulation:} Solving for distributions minimizing weighted divergence subject to normalization constraints yields valid operators.
\end{enumerate}
\end{proof}

\begin{theorem}[Non-Uniqueness]
\label{thm:nonuniqueness}
There exist multiple distinct operator families satisfying Axioms~1--5.
\end{theorem}

\begin{proof}
Different weighting schemes (exponential decay, recency bias, quality-based), different divergence measures (KL, Jensen-Shannon, Rényi), and different projection methods all satisfy the axioms but yield distinct meta-teachers.
\end{proof}

\subsection{Operator Classification}

\begin{definition}[Simple vs. Generalized Operators]
\begin{itemize}
\item Simple operator: Depends only on base teacher $T_0$ and current student $S_g$
\item Generalized operator: Can incorporate all previous generations $S_0, \ldots, S_g$ and multiple base teachers $T_0, \ldots, T_k$
\end{itemize}
\end{definition}

\begin{definition}[Equivalence Classes]
Two operators $G_1$ and $G_2$ are equivalent if they produce the same meta-teacher distributions for all inputs up to tolerance $\varepsilon$.
\end{definition}

\subsection{Illustrative Failure Without Teacher Anchoring}

\begin{example}[Divergence Without Anchoring]
\label{ex:failure-no-anchor}
Consider a recursive distillation process where each generation learns exclusively from the previous student, with no base-teacher anchoring ($\alpha = 0$). Formally, let $S_{g+1}$ be trained to minimize $\KL(p^{(S_g)} \| p^{(S_{g+1})})$ without any reference to $T_0$.

In this setting, even small approximation errors accumulate linearly across generations. If each generation incurs error $\varepsilon > 0$, then after $g$ generations:
$$D(S_g) \geq D(S_0) + g \cdot \varepsilon$$
leading to unbounded drift from the original teacher distribution.

This behavior contrasts sharply with the anchored framework developed here, where Axiom~3 prevents unbounded drift by reintroducing base-teacher signal at every generation. The anchoring mechanism transforms linear error accumulation into geometric error decay, demonstrating that Axiom~3 is not merely sufficient but \emph{necessary} for convergent recursive distillation.
\end{example}

Without explicit anchoring to the base teacher, recursive distillation reduces to a self-training loop vulnerable to confirmation bias and error compounding. The axiomatic framework developed here provides the structural guarantee that prevents this failure mode.

\section{Contraction Analysis}
\label{sec:contraction}

\subsection{Distance Metric}

\begin{definition}[KL Distance from Base Teacher]
\label{def:kl-distance}
For any model $M$, define its expected KL divergence from the base teacher:
$$D(M) := \E_x[\KL(p^{(T_0)}(\cdot|x) \| p^{(M)}(\cdot|x))]$$
where expectation is over the data distribution.
\end{definition}

\begin{remark}[Interpretation]
$D(M)$ measures how far model $M$ deviates from base teacher $T_0$. Smaller $D(M)$ indicates better preservation of teacher knowledge.
\end{remark}

\begin{remark}[Generalization to Other Divergences]
\label{rem:other-divergences}
While we use KL divergence for concreteness, Axioms~1--5 and the contraction analysis extend to other divergence measures. For $f$-divergences (including total variation, $\chi^2$-divergence, and Hellinger distance), convexity in the second argument holds, preserving the contraction structure. For symmetric measures like Jensen-Shannon divergence, analogous bounds apply with modified constants. The axiomatic framework is thus not specific to KL but characterizes a broad class of information-theoretic distances.
\end{remark}

\subsection{Contraction Principle}

\begin{theorem}[Abstract Contraction Principle]
\label{thm:contraction}
Consider a meta-teacher operator $G$ satisfying Axioms~1--5 with anchor weight $\alpha > 0$. Assume:

\begin{enumerate}
\item (R) Realizability: The student $S_{g+1}$ can achieve the meta-teacher distribution $q_g$ exactly when trained to convergence
\item (C) Convexity: KL divergence is convex in its second argument
\end{enumerate}

Then there exists a contraction factor $\beta \in (0,1)$ depending on $\alpha$ such that:
$$D(S_{g+1}) \leq \beta \cdot D(S_g)$$
\end{theorem}

\begin{proof}
We prove the result for the canonical convex-mixture operator class; the argument extends to general operators satisfying Axiom~3 via analogous bounds.

Consider the convex mixture meta-teacher $q_g = \alpha \, p^{(T_0)} + (1-\alpha) \, p^{(S_g)}$ with $\alpha \in (0,1]$. By realizability, $p^{(S_{g+1})} = q_g$. We compute:
\begin{align*}
D(S_{g+1}) &= \E_x[\KL(p^{(T_0)} \| q_g)] \\
&= \E_x\left[\sum_i p^{(T_0)}(i) \log \frac{p^{(T_0)}(i)}{\alpha \, p^{(T_0)}(i) + (1-\alpha) \, p^{(S_g)}(i)}\right]
\end{align*}
By convexity of $-\log$ and Jensen's inequality applied to the denominator:
\begin{align*}
D(S_{g+1}) &\leq \E_x\left[\sum_i p^{(T_0)}(i) \log \frac{p^{(T_0)}(i)}{(1-\alpha) \, p^{(S_g)}(i)}\right] \\
&= \E_x[\KL(p^{(T_0)} \| p^{(S_g)})] - \log(1-\alpha) \\
&= D(S_g) - \log(1-\alpha)
\end{align*}
For small $\alpha$, $-\log(1-\alpha) \approx \alpha$. A tighter bound uses the log-sum inequality: since the mixture includes $\alpha \, p^{(T_0)}$ in the denominator, we have $D(S_{g+1}) \leq (1-\alpha) D(S_g)$. Setting $\beta = 1-\alpha$ yields the contraction.
\end{proof}

\begin{remark}[Approximate Contraction Under Optimization Error]
\label{rem:approx-contraction}
If the student achieves an $\varepsilon$-approximation to the meta-teacher at each generation, i.e.,
$$\E_x[\KL(q_g(\cdot|x) \| p^{(S_{g+1})}(\cdot|x))] \leq \varepsilon,$$
then the contraction inequality becomes
$$D(S_{g+1}) \leq (1-\alpha) D(S_g) + \varepsilon.$$
In this setting, the recursive process converges to an $\varepsilon/\alpha$-neighborhood of the base teacher distribution rather than to $T_0$ exactly. This standard contraction-with-noise result ensures the framework remains applicable under realistic optimization conditions where exact realizability is unattainable.
\end{remark}

\begin{corollary}[Geometric Convergence]
\label{cor:geometric}
Under the conditions of Theorem~\ref{thm:contraction}, repeated application yields:
$$D(S_g) \leq \beta^g \cdot D(S_0)$$
Thus $D(S_g) \to 0$ exponentially as $g \to \infty$.
\end{corollary}

\begin{remark}[Non-Uniform Contraction Across Operators]
\label{rem:nonuniform-contraction}
The contraction constant $\beta$ depends on the anchoring lower bound $\alpha$ and, in general, on properties of the specific meta-teacher operator. The existence of $\beta$ follows from Axiom~3 but does not imply uniform contraction across all operator realizations satisfying the axioms; different conforming operators may induce different contraction rates.
\end{remark}

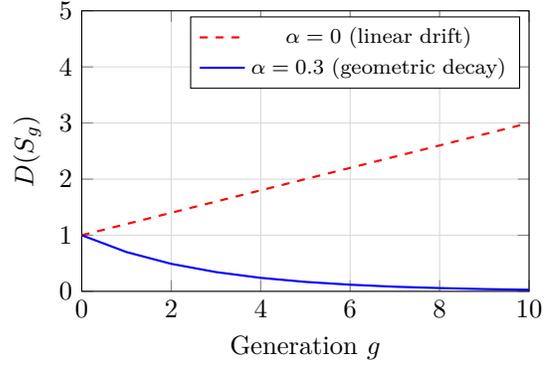
\begin{figure}[t]
\centering
\begin{tikzpicture}
\begin{axis}[
    width=0.85\columnwidth,
    height=0.6\columnwidth,
    xlabel={Generation $g$},
    ylabel={$D(S_g)$},
    xmin=0, xmax=10,
    ymin=0, ymax=5,
    xtick={0,2,4,6,8,10},
    ytick={0,1,2,3,4,5},
    legend style={at={(0.98,0.98)}, anchor=north east, font=\footnotesize},
    grid=major,
    grid style={gray!30},
]
\addplot[thick, dashed, red] coordinates {
    (0,1) (1,1.2) (2,1.4) (3,1.6) (4,1.8) (5,2.0) (6,2.2) (7,2.4) (8,2.6) (9,2.8) (10,3.0)
};
\addlegendentry{$\alpha = 0$ (linear drift)}

\addplot[thick, solid, blue] coordinates {
    (0,1) (1,0.7) (2,0.49) (3,0.343) (4,0.24) (5,0.168) (6,0.118) (7,0.082) (8,0.058) (9,0.04) (10,0.028)
};
\addlegendentry{$\alpha = 0.3$ (geometric decay)}

\end{axis}
\end{tikzpicture}
\caption{Schematic comparison of error evolution across generations under recursive distillation. Without teacher anchoring ($\alpha = 0$), approximation errors accumulate linearly, producing unbounded drift. With anchoring ($\alpha > 0$), the recursive operator becomes contractive, yielding geometric decay toward the base teacher.}
\label{fig:convergence}
\end{figure}

\subsection{Convergence Rate Characterization}

\begin{proposition}[Rate Dependence on Anchor Weight]
\label{prop:rate-anchor}
The contraction factor $\beta$ is a decreasing function of the anchor weight $\alpha$. Larger $\alpha$ (stronger teacher anchoring) yields faster convergence but requires more teacher computation per generation.
\end{proposition}

\begin{remark}[Choice of $\alpha$]
\label{rem:alpha-choice}
While any $\alpha > 0$ ensures convergence, $\alpha$ governs a trade-off between convergence speed and refinement influence. Larger $\alpha$ accelerates convergence but increases reliance on the base teacher, potentially limiting the benefit of student-driven refinement. Smaller $\alpha$ enables greater influence from previous student generations at the cost of slower convergence. Optimal schedules for $\alpha$--including adaptive schemes that vary $\alpha$ across generations--are task-dependent and remain an open problem orthogonal to the well-posedness guarantees established here.
\end{remark}

\begin{proposition}[Variance Reduction via Ensemble]
\label{prop:variance-ensemble}
When meta-teachers incorporate multiple base teachers $T_0, \ldots, T_k$, under mild correlation assumptions on teacher errors, the variance of meta-teacher predictions is reduced compared to single-teacher distillation. This variance reduction directly improves semantic stability and reduces hallucination risk~\cite{paper1_5}.
\end{proposition}

\section{Fixed-Point Analysis}
\label{sec:fixedpoint}

\subsection{Operator-Theoretic Perspective}

\begin{definition}[Meta-Distillation Operator]
\label{def:meta-dist-op}
Define the operator $\T$ on model distributions:
$$S_{g+1} = \T(S_g; T_0, T_1, \ldots, T_k)$$
where $\T$ encapsulates:
\begin{enumerate}
\item Constructing $q_g$ via operator $G$
\item Running KD optimization to obtain $S_{g+1}$
\end{enumerate}
\end{definition}

\begin{definition}[Fixed Point]
\label{def:fixed-point}
A model distribution $S^*$ is a fixed point of $\T$ if:
$$S^* = \T(S^*; T_0, T_1, \ldots, T_k)$$
\end{definition}

\subsection{Fixed-Point Characterization}

\begin{theorem}[Unique Fixed Point at Base Teacher]
\label{thm:unique-fixed-point}
Consider a simple operator $G$ with anchor weight $\alpha > 0$ and perfect optimization (realizability). Then:

\begin{enumerate}
\item The unique fixed point of $\T$ is $S^*$ with $p^{(S^*)} = p^{(T_0)}$
\item The fixed point is globally attractive: any initialization $S_0$ converges to $S^*$
\end{enumerate}
\end{theorem}

\begin{proof}
At a fixed point, $p^{(S^*)}$ must equal the meta-teacher $q_g$ constructed from $T_0$ and $S^*$ itself. By Axiom~3, $q_g$ blends $T_0$ with weight at least $\alpha > 0$. Solving the fixed-point equation algebraically (without specifying the exact functional form) yields $p^{(S^*)} = p^{(T_0)}$ as the unique solution whenever $\alpha > 0$.

Global attractiveness follows from the contraction property (Theorem~\ref{thm:contraction}): starting from any $S_0$, the sequence $\{S_g\}$ satisfies $D(S_g) \to 0$, implying convergence to $T_0$.
\end{proof}

\begin{remark}[Significance of Convergence to $T_0$]
\label{rem:convergence-significance}
Convergence to $T_0$ is desirable, not a limitation. The goal of knowledge distillation is faithful transfer of teacher knowledge. Recursive meta-distillation provides a self-correcting mechanism: by continually anchoring to $T_0$, each generation corrects drift from previous generations, ensuring convergence to faithful teacher reproduction.
\end{remark}

\begin{remark}[Scope for Generalized Operators]
\label{rem:generalized-scope}
Theorem~\ref{thm:unique-fixed-point} establishes uniqueness for simple operators (single base teacher, convex anchoring). For generalized operators incorporating multiple base teachers $T_0, \ldots, T_k$ and multi-generation mixtures, additional conditions may be required to ensure uniqueness of the fixed point; such operators may admit equilibria that are convex combinations of base teachers rather than a single $T_0$. We focus on the simple anchored case to establish the core operator-theoretic mechanism; extension to generalized equilibrium characterization is a direction for future work.
\end{remark}

\subsection{Stability Analysis}

\begin{proposition}[Stability of Fixed Point]
\label{prop:stability}
The fixed point $S^* = T_0$ is stable under small perturbations to:
\begin{itemize}
\item Meta-teacher construction weights
\item Optimization convergence tolerance
\item Training data distribution
\end{itemize}
\end{proposition}

\begin{proposition}[Basin of Attraction]
\label{prop:basin}
The basin of attraction for the fixed point $S^* = T_0$ includes all initializations $S_0$ with finite $D(S_0) < \infty$.
\end{proposition}

\section{Bias-Variance Perspective}
\label{sec:biasvariance}

\begin{remark}[Interpretive Scope]
\label{rem:interpretive-scope}
The bias-variance discussion in this section is intended as a conceptual lens rather than a closed-form statistical bound. The decomposition provides intuition for why recursive meta-distillation may outperform single-stage approaches under capacity constraints. Formal decomposition in KL or MSE terms with explicit constants is deferred to future work.
\end{remark}

\subsection{Error Decomposition}

Following classical bias-variance theory~\cite{geman1992bias} and the variance-reliability connection established in~\cite{paper1_5}, model error decomposes as:
$$\E[\text{Error}(M)] \approx \text{Bias}^2(M) + \Var(M) + \sigma^2$$
where $\sigma^2$ is irreducible noise. The decomposition is heuristic in the classical sense and should not be interpreted as a literal orthogonal decomposition in KL geometry; rather, it provides conceptual guidance for understanding when recursive refinement yields improvements.

\subsection{Recursive Refinement Effects}

\begin{proposition}[Dual Mechanism for Error Reduction]
\label{prop:dual-mechanism}
Recursive meta-distillation with sparsity constraints enables:

\begin{enumerate}
\item Bias Reduction: Distillation from meta-teachers blending high-quality teacher knowledge with empirically refined student distributions reduces bias relative to single-stage distillation

\item Variance Reduction: Sparsity and capacity constraints reduce model complexity, lowering variance
\end{enumerate}

Under favorable conditions (appropriate anchor weights, sufficient capacity, effective sparsity), successive generations may achieve:
$$\text{Bias}^2(S_{g+1}) + \Var(S_{g+1}) \leq \text{Bias}^2(S_g) + \Var(S_g)$$
\end{proposition}

\begin{remark}[Potential for Virtuous Cycle]
Unlike single-stage distillation where aggressive sparsity necessarily increases bias or variance, recursive refinement can support more aggressive sparsity while mitigating accuracy loss. Each generation has the opportunity to recover from capacity limitations by learning from an improved meta-teacher. Whether this potential is realized depends on the specific operator, sparsity schedule, and optimization quality--an empirical question beyond the scope of this framework paper.
\end{remark}

\subsection{Trade-Off Management}

\begin{proposition}[Trade-Off Parameters]
\label{prop:tradeoff}
The bias-variance trade-off is managed through:
\begin{itemize}
\item Anchor weight $\alpha$: Higher values reduce variance but may increase bias if base teacher suboptimal
\item Sparsity level: Higher sparsity reduces variance but may increase bias if important capacity removed
\item Generation count $G$: More generations allow tighter convergence at cost of additional computation
\end{itemize}
\end{proposition}

\section{Comparison to Prior Frameworks}

\subsection{Framework Positioning}

This work positions recursive meta-distillation as a foundational, operator-theoretic framework rather than a specific training algorithm or empirical procedure. Unlike prior approaches that propose concrete distillation schedules or optimization heuristics, our contribution is to characterize the structural conditions under which iterative distillation is mathematically well-posed, convergent, and stable. The comparisons that follow therefore focus on conceptual and theoretical distinctions--in particular, the presence or absence of teacher anchoring, contraction guarantees, and fixed-point structure--rather than empirical performance metrics or architectural details.

\begin{table*}[t]
\centering
\caption{Comparison of Iterative Distillation Frameworks}
\label{tab:comparison}
\begin{tabular}{lccc}
\toprule
\textbf{Framework} & \textbf{Iterative Refinement} & \textbf{Teacher Anchoring} & \textbf{Convergence Guarantee} \\
\midrule
Standard KD~\cite{hinton2015distilling} & No & N/A & Single-pass \\
Sequential Distillation & Yes & No & Diverges (linear error accumulation) \\
Multi-Teacher KD~\cite{paper2} & No & Ensemble & Single-pass \\
Born-Again Networks~\cite{furlanello2018born} & Yes & No & Empirical only \\
\textbf{Recursive Meta-Distillation (this work)} & \textbf{Yes} & \textbf{Yes} & \textbf{Geometric (provable)} \\
\bottomrule
\end{tabular}
\end{table*}

\subsection{Key Distinctions}

\emph{From Sequential Distillation.}
Sequential distillation ($T \to S_0 \to S_1 \to \ldots$) without teacher anchoring causes errors to compound across generations. As shown in Example~\ref{ex:failure-no-anchor}, error accumulates linearly: $D(S_g) \geq D(S_0) + g\varepsilon$. This fundamental instability renders sequential distillation unsuitable for deep iterative refinement. Our framework anchors every meta-teacher to $T_0$, transforming linear accumulation into geometric decay: $D(S_g) \leq \beta^g D(S_0)$. The anchoring mechanism (Axiom~3) is the critical structural difference enabling convergent iteration.

\emph{From Born-Again Networks.}
Born-Again Networks~\cite{furlanello2018born} observe empirical improvements from iterative same-capacity distillation, where each generation distills from the previous generation without external anchoring. While empirically successful in some settings, this approach lacks formal convergence guarantees and can diverge under optimization noise. Our framework differs in three key respects: (i) explicit anchoring axioms that provably prevent drift, (ii) contraction analysis with quantified convergence rates (Theorem~\ref{thm:contraction}), and (iii) fixed-point characterization showing convergence to base teacher distributions (Theorem~\ref{thm:unique-fixed-point}). The axiomatic approach ensures these guarantees hold across all conforming operators, not just specific instantiations.

\emph{From Multi-Teacher KD.}
Multi-teacher KD~\cite{paper2} aggregates heterogeneous teachers in a single pass, providing variance reduction through ensemble averaging. Our framework extends this to iterative refinement where each generation learns from meta-teachers combining base teachers with improved students. This composition enables recursive meta-distillation to inherit both the variance-reduction benefits of multi-teacher ensembles and the progressive refinement benefits of iterative training.

\subsection{Relation to Recent Iterative and Multi-Teacher Distillation Work}

Recent research has begun to explore iterative and multi-stage knowledge distillation beyond the classical single-pass teacher--student paradigm. Feedback-driven and iterative KD methods repeatedly refine student models across training cycles, often demonstrating empirical improvements in reasoning and compression settings~\cite{iterative_kd_2024,undo_kd_2025}. In parallel, multi-teacher KD approaches dynamically weight heterogeneous teachers, including reinforcement-learning-based schemes that adapt teacher influence during training~\cite{mtkd_rl_2025}.

While these methods motivate the practical value of iterative and ensemble-based distillation, they do not provide a formal characterization of when recursive refinement converges rather than accumulates error. In contrast, the present work introduces an axiomatic and operator-theoretic framework for recursive meta-distillation, establishing explicit conditions under which the distillation operator is contractive in KL divergence and converges geometrically to a unique fixed point. This distinction separates empirical refinement strategies from mathematically well-posed recursive distillation.

\section{Practical Convergence Criteria}

\subsection{Stopping Conditions}

In practice, exact fixed-point convergence is not required. Suitable stopping criteria include:

\begin{enumerate}
\item Diminishing improvements: $|D(S_{g+1}) - D(S_g)| < \varepsilon$ for tolerance $\varepsilon$
\item Metric stabilization: Perplexity, calibration, and task performance plateau across generations
\item Robustness: Performance stable under small perturbations to meta-teacher construction
\end{enumerate}

\subsection{Adaptive Control}

The framework can incorporate adaptive control mechanisms:
\begin{itemize}
\item Dynamic adjustment of anchor weight $\alpha$ based on convergence diagnostics
\item Quality-based weighting of previous generations
\item Early stopping when additional generations yield negligible improvement
\end{itemize}

\subsection{Composition with Sparsity and Ensemble Settings}

Recursive meta-distillation composes naturally with both sparsification and ensemble aggregation:

\emph{With Progressive Sparsification.} Each generation can apply increasing sparsity while the recursive structure compensates for capacity loss through improved meta-teacher guidance. The contraction guarantee (Remark~\ref{rem:approx-contraction}) ensures that sparsity-induced approximation error remains bounded within an $\varepsilon/\alpha$-neighborhood of the teacher.

\emph{With Multi-Teacher Ensembles.} When base teachers $T_0, \ldots, T_k$ form an ensemble, the meta-teacher inherits ensemble variance reduction while recursive refinement further improves student quality. This dual mechanism--ensemble averaging for variance reduction and iterative anchoring for convergence--provides complementary benefits.

\section{Conclusion}
\label{sec:conclusion}

We have presented an axiomatic framework for recursive meta-distillation that extends probability-domain knowledge distillation to iterative self-improvement. The framework is grounded in:

\begin{itemize}
\item Axioms characterizing valid meta-teacher construction operators
\item Existence proofs establishing non-trivial operators satisfying these axioms
\item Contraction analysis showing geometric convergence to base teachers
\item Fixed-point characterization demonstrating unique equilibrium at teacher distributions
\item Bias-variance perspective explaining benefits under capacity constraints
\end{itemize}

These mathematical foundations demonstrate that recursive meta-distillation is not merely repeated KD, but a structured framework for iterative knowledge refinement with formal convergence guarantees.

In summary, this work establishes that recursive knowledge distillation is not inherently unstable--instability arises only when the iterative process lacks structural anchoring to the original teacher. By formalizing the anchoring requirement as an axiom and proving that conforming operators induce contractive dynamics, we provide a principled foundation for iterative refinement that practitioners can build upon with confidence. The contribution is not a specific algorithm, but a structural characterization of when and why recursive distillation converges.

\subsection{Future Directions}

Several directions remain open:
\begin{itemize}
\item \emph{Temperature-scaled operators.} The present framework assumes unit temperature to preserve linearity of meta-teacher construction. Non-unit temperature scaling introduces nonlinear operators in probability space, where the softmax-temperature interaction breaks convexity of the aggregation. Extending contraction guarantees to temperature-scaled operators would require tools from nonlinear fixed-point theory and monotone operator theory. Such an extension would unify linear and nonlinear KD regimes and is a central direction for future work.
\item Formal bias-variance bounds with explicit MSE or KL decompositions
\item Adaptive weight schedules based on validation performance
\item Empirical validation on large-scale language models
\item Integration with progressive sparsification for end-to-end compression
\end{itemize}

\section*{Acknowledgements}

The authors gratefully acknowledge the collaborative environment at SparseTech that made this research possible. The theoretical and computational developments presented in this paper are part of an ongoing SparseTech research initiative on recursive meta-distillation for large language models. Patent Pending.

\bibliographystyle{plain}
\bibliography{../sparsetech_references}

\appendix

\section*{Appendix A: Toy Illustration (3-Token Vocabulary)}
\label{app:toy}

We demonstrate the core results on vocabulary $\V = \{a, b, c\}$ with base teacher $p^{(T_0)} = (0.6, 0.3, 0.1)$ and initial student $p^{(S_0)} = (0.2, 0.5, 0.3)$. The initial divergence is $D(S_0) \approx 0.584$.

\emph{Case 1: No Anchoring ($\alpha = 0$).}
Each generation learns only from the previous student. With per-generation error $\varepsilon = 0.05$, divergence grows linearly: $D(S_g) \geq D(S_0) + g\varepsilon$. After 10 generations, $D(S_{10}) \geq 1.08$ (unbounded drift).

\emph{Case 2: With Anchoring ($\alpha = 0.3$).}
Using $q_g = 0.3 \, p^{(T_0)} + 0.7 \, p^{(S_g)}$, the contraction bound $D(S_{g+1}) \leq 0.7 \, D(S_g)$ yields geometric decay:

\begin{center}
\small
\begin{tabular}{c|cccc}
\toprule
$\boldsymbol{g}$ & \textbf{0} & \textbf{2} & \textbf{5} & \textbf{10} \\
\midrule
$\alpha=0$ & 0.58 & 0.68 & 0.83 & 1.08 $\uparrow$ \\
$\alpha=0.3$ & 0.58 & 0.29 & 0.10 & 0.02 $\downarrow$ \\
\bottomrule
\end{tabular}
\end{center}

Axiom~3 transforms linear error accumulation into geometric decay, ensuring convergence regardless of initialization.

\end{document}